\documentclass[a4paper,UKenglish]{lipics-v2016}
 
\usepackage{microtype}

\def\showauthornotes{0}

\def\showkeys{0}
\def\showdraftbox{0}

\def\usemicrotype{1}
\def\showfixme{1}

\def\writemode{0}

\title{Agnostic Learning by Refuting}

\author[1]{Pravesh K. Kothari}
\author[2]{Roi Livni}
\affil[1]{Princeton University and Institute of Advanced Study, Princeton, NJ\\
  \texttt{kothari@cs.princeton.edu}}
\affil[2]{Princeton University, Princeton, NJ\\
  \texttt{rlivni@cs.princeton.edu}}
\authorrunning{P.\,K. Kothari and R. Livni} 

\Copyright{P.K Kothari and R. Livni}


\EventEditors{John Q. Open and Joan R. Acces}
\EventNoEds{2}
\EventLongTitle{42nd Conference on Very Important Topics (CVIT 2016)}
\EventShortTitle{CVIT 2016}
\EventAcronym{CVIT}
\EventYear{2016}
\EventDate{December 24--27, 2016}
\EventLocation{Little Whinging, United Kingdom}
\EventLogo{}
\SeriesVolume{42}
\ArticleNo{23}

\usepackage{etex}

\usepackage[l2tabu, orthodox]{nag}

\usepackage{xspace,enumerate}

\usepackage[dvipsnames]{xcolor}

\usepackage[T1]{fontenc}

\usepackage{mathtools}

\usepackage{amsthm}

\newtheorem*{theorem*}{Theorem}

\newtheorem*{proposition*}{Proposition}
\newtheorem*{lemma*}{Lemma}
\newtheorem*{conjecture*}{Conjecture}
\newtheorem{fact}[theorem]{Fact}
\newtheorem*{fact*}{Fact}

\newtheorem*{hypothesis*}{Hypothesis}

\theoremstyle{definition}

\theoremstyle{remark}

\newtheorem*{claim*}{Claim}
\newtheorem*{remark*}{Remark}

\newtheorem*{observation*}{Observation}

\newcommand{\ignore}[1]{}
\ignore{

\ifnum\writemode=1
\usepackage[
letterpaper,
top=1.2in,
bottom=1.2in,
left=1in,
right=1in]{geometry}

\pagestyle{empty}
\fi

\ifnum\writemode=0
\usepackage[
letterpaper,
top=0.7in,
bottom=0.9in,
left=1in,
right=1in]{geometry}
\fi

\usepackage{newpxtext} 
\usepackage{textcomp} 
\usepackage[varg,bigdelims]{newpxmath}
\usepackage{bm} 
\linespread{1.1}
\let\mathbb\varmathbb

\ifnum\showkeys=1
\usepackage[color]{showkeys}
\fi
}

\usepackage[capitalise,nameinlink]{cleveref}
\crefname{lemma}{Lemma}{Lemmas}
\crefname{definition}{Definition}{Definitions}

\usepackage{nicefrac}

\ifnum\usemicrotype=1
\usepackage{microtype}
\fi

\ifnum\showauthornotes=1
\newcommand{\Authornote}[2]{{\sffamily\small\color{red}{[#1: #2]}}}
\newcommand{\Authornotecolored}[3]{{\sffamily\small\color{#1}{[#2: #3]}}}
\newcommand{\Authorcomment}[2]{{\sffamily\small\color{gray}{[#1: #2]}}}
\newcommand{\Authorstartcomment}[1]{\sffamily\small\color{gray}[#1: }

\newcommand{\Authorfnote}[2]{\footnote{\color{red}{#1: #2}}}
\newcommand{\Authorfixme}[1]{\Authornote{#1}{\textbf{??}}}
\newcommand{\Authormarginmark}[1]{\marginpar{\textcolor{red}{\fbox{\Large #1:!}}}}
\else
\newcommand{\Authornote}[2]{}
\newcommand{\Authornotecolored}[3]{}
\newcommand{\Authorcomment}[2]{}
\newcommand{\Authorstartcomment}[1]{}

\newcommand{\Authorfnote}[2]{}
\newcommand{\Authorfixme}[1]{}
\newcommand{\Authormarginmark}[1]{}
\fi

\newcommand{\one}[1]{\mathbf{1}\left[#1\right]}

\definecolor{forestgreen(traditional)}{rgb}{0.0, 0.27, 0.13}

\ifnum\showfixme=0

\fi

\usepackage{boxedminipage}

\newcommand{\Esymb}{\mathbb{E}}
\newcommand{\Psymb}{\mathbb{P}}

\DeclareMathOperator*{\E}{\Esymb}

\DeclareMathOperator*{\ProbOp}{\Psymb}

\renewcommand{\Pr}{\ProbOp}

\newcommand{\textparen}[1]{\text{(#1)}}

\ifx\because\undefined
\newcommand{\because}[1]{\textparen{because #1}}
\else
\renewcommand{\because}[1]{\textparen{because #1}}
\fi

\newcommand\bdot\bullet

\ifx\mathds\undefined 

\else
}
\fi

\DeclareMathOperator{\poly}{poly}

\newcommand{\N}{\mathbb N}
\newcommand{\R}{\mathbb R}
\newcommand{\C}{\mathbb C}

\newcommand{\cA}{\mathcal A}

\newcommand{\cD}{\mathcal D}

\newcommand{\cR}{\mathcal R}

\renewcommand{\leq}{\leqslant}

\renewcommand{\geq}{\geqslant}

\ifnum\showdraftbox=1

\else

\fi

\let\epsilon=\varepsilon

\numberwithin{equation}{section}

\newcommand\MYcurrentlabel{xxx}

\newcommand{\MYstore}[2]{%
  \global\expandafter \def \csname MYMEMORY #1 \endcsname{#2}%
}

\newcommand{\MYload}[1]{%
  \csname MYMEMORY #1 \endcsname%
}

\newcommand{\MYnewlabel}[1]{%
  \renewcommand\MYcurrentlabel{#1}%
  \MYoldlabel{#1}%
}

\newcommand{\MYdummylabel}[1]{}

\newcommand{\torestate}[1]{%
  \let\MYoldlabel\label%
  \let\label\MYnewlabel%
  #1%
  \MYstore{\MYcurrentlabel}{#1}%
  \let\label\MYoldlabel%
}

\newcommand{\restatetheorem}[1]{%
  \let\MYoldlabel\label
  \let\label\MYdummylabel
  \begin{theorem*}[Restatement of \prettyref{#1}]
    \MYload{#1}
  \end{theorem*}
  \let\label\MYoldlabel
}

\newcommand{\restatelemma}[1]{%
  \let\MYoldlabel\label
  \let\label\MYdummylabel
  \begin{lemma*}[Restatement of \prettyref{#1}]
    \MYload{#1}
  \end{lemma*}
  \let\label\MYoldlabel
}

\newcommand{\restateprop}[1]{%
  \let\MYoldlabel\label
  \let\label\MYdummylabel
  \begin{proposition*}[Restatement of \prettyref{#1}]
    \MYload{#1}
  \end{proposition*}
  \let\label\MYoldlabel
}

\newcommand{\restatefact}[1]{%
  \let\MYoldlabel\label
  \let\label\MYdummylabel
  \begin{fact*}[Restatement of \prettyref{#1}]
    \MYload{#1}
  \end{fact*}
  \let\label\MYoldlabel
}

\newcommand{\restate}[1]{%
  \let\MYoldlabel\label
  \let\label\MYdummylabel
  \MYload{#1}
  \let\label\MYoldlabel
}

\let\origparagraph\paragraph
\renewcommand{\paragraph}[1]{\origparagraph{#1.}}

\allowdisplaybreaks

\usepackage{paralist}

\usepackage{comment}

\usepackage{braket}

\EventEditors{Anna R. Karlin}
\EventNoEds{1}
\EventLongTitle{9th Innovations in Theoretical Computer Science Conference (ITCS
2018)}
\EventShortTitle{ITCS 2018}
\EventAcronym{ITCS}
\EventYear{2018}
\EventDate{January 11--14, 2018}
\EventLocation{Cambridge, MA, USA}
\EventLogo{}
\SeriesVolume{94}
\ArticleNo{55}

\usepackage{relsize}
\usepackage[font=footnotesize]{caption}
\usepackage{appendix,color}

\DeclareUrlCommand\email{}

\DeclareMathOperator{\zo}{\{0,1\}}

\renewcommand{\C}{\mathcal{C}}

\newcommand{\noise}{\mathsf{noise}}
\newcommand{\unclear}{\mathsf{structure}}
\newcommand{\on}{\{\pm 1\}}
\newcommand{\U}{\mathcal{U}}

\newcommand{\cor}{\mathrm{cor}}
\begin{document}

\maketitle

\begin{abstract}
The sample complexity of learning a Boolean-valued function class is precisely characterized by its Rademacher complexity. This has little bearing, however, on the sample complexity of \emph{efficient} agnostic learning. 

We introduce \emph{refutation complexity}, a natural computational analog of Rademacher complexity of a Boolean concept class and show that it exactly characterizes the sample complexity of \emph{efficient} agnostic learning. Informally, refutation complexity of a class $\C$ is the minimum number of  example-label pairs required to efficiently distinguish between the case that the labels correlate with the evaluation of some member of $\C$ (\emph{structure}) and the case where the labels are i.i.d. Rademacher random variables (\emph{noise}). The easy direction of this relationship was implicitly used in the recent framework for improper PAC learning lower bounds of Daniely and co-authors \cite{DBLP:conf/stoc/Daniely16,DBLP:conf/stoc/DanielyLS14,daniely2016complexity} via connections to the hardness of refuting random constraint satisfaction problems. Our work can be seen as making the relationship between agnostic learning and refutation implicit in their work into an explicit equivalence. 
In a recent, independent work, Salil Vadhan \cite{vadhan2017learning} discovered a similar relationship between refutation and PAC-learning in the realizable (i.e. noiseless) case.

 \end{abstract}

\section{Introduction}
Statistical complexity characterizes the information theoretic threshold for the amount of data required for any supervised learning task. However, the amount of data required for \emph{efficient} learning, whenever it is possible, can be significantly different from the statistical complexity.  For example, algorithms based on polynomial regression (\cite{DBLP:journals/siamcomp/KalaiKMS08,DBLP:conf/colt/KaneKM13,Klivans:2008:LGC:1470582.1470603}) guarantee efficient (improper, i.e. return a hypothesis not necessarily in the target class) learning while using data that is a polynomial factor larger than the statistical complexity. There is a systematic effort to study the trade-offs between computational and statistical complexity \cite{DBLP:journals/corr/abs-1304-0828,Chandrasekaran26032013} and a growing body of work has provided explicit examples \cite{DBLP:journals/siamcomp/DecaturGR99,DBLP:conf/nips/DanielyLS13,DBLP:journals/corr/BarakM15} of natural settings where efficient learning provably requires data that is at least a polynomial factor larger than the statistical complexity under some plausible complexity theoretic assumptions. 

In the light of the above work, we focus on obtaining a simple and useful characterization of the sample complexity of \emph{efficient} supervised learning. There's a simple and elegant characterization of the statistical complexity of learning in terms of the Rademacher complexity \footnote{The related notion of VC Dimension of $\C$ characterizes the data required to learn $\C$ over \emph{worst-case} distributions.}. In this note, we give a natural analog of Rademacher complexity that precisely characterizes the amount of data required for \emph{efficient agnostic} (i.e. noisy, see Definition \ref{def:agnostic-learning}) learning.




{}




For a class $\C$ of concepts on $\R^n$, any distribution $\cD$ on $\R^n$, the \emph{Rademacher Complexity} of $\C$, $\cR_m(\C)$ is the following quantity: 
\begin{align}
\mathcal{R}_m(\C) = \E_{\begin{subarray}{c} x_i \sim_{i.i.d.} \cD\\ 1\leq i \leq m \end{subarray}}\left[\E_{\begin{subarray}{c}\sigma_i \sim_{i.i.d.} \{\pm 1\} \\ 1 \leq i \leq m \end{subarray} }\left[\frac{1}{m}\sup_{c\in \C}\sum_{i=1}^m  \sigma_{i}c(x_i)\right]\right].
\end{align}
Classical results \cite{BartlettM02} establish that $\cR_m(\C) = \epsilon$ if and only if there's an algorithm to learn $\C$ over $\cD$ with error at most $\epsilon$ with $\Theta(m)$ samples, thus characterizing the sample-complexity of $\epsilon$-error agnostic learning. 

In this note, we propose a natural computational analog of Rademacher complexity, called as the \emph{Refutation complexity} and show that it exactly determines the sample complexity of efficient agnostic learning. Given random labeled examples $\{(x_i, y_i)\}_{i \leq m}$ where $x_i$s are chosen i.i.d. according to $\cD$, we define the problem of \emph{refutation} as the task of distinguishing between the following two cases:
\begin{enumerate}[(a)]
 \item \textbf{Structure: } $\{(x_i,y_i)\}_{i \leq m}$ are i.i.d. from some distribution $\cD'$ with marginal on $x_i$s being $\cD$ and $\E_{(x,y) \sim \cD'}[c(x)y] = \Omega(1).$ That is, the given example-label pairs come from a distribution that correlates with some $c\in \C$, and 
 \item \textbf{Noise: }$y_i$s are uniform and independent Rademacher random variables. \label{s}
 \end{enumerate}
We define refutation complexity of $\C$ with respect to the distribution $\cD$ at a running time of $T(n)$ as the smallest $m$ for which there's a $T(n)$-time test for distinguishing between structure and noise cases above. 

To motivate this definition, observe that we can interpret the statistical complexity (via the connection to Rademacher complexity outlined above) of $\C$ over $\cD$ as the smallest $m$ for which no concept in $\C$ correlates with purely random noise (the i.i.d. draws from $\{ \pm 1\}.$) Thus, if the Rademacher complexity of $\C$ on $\cD$ with $m$ samples is small enough, then, given random labeled examples $\{ (x_i, y_i)\}_{i \leq m}$, we can (via an inefficient procedure) distinguish between the above two cases by computing the largest correlation of any $c \in \C$ when evaluated at $x_i$s with the $y_i$s. Thus, we can equivalently define statistical complexity as the smallest $m$ for which the above structure vs noise test succeeds. Thus, refutation complexity can be seen as a computational analog of Rademacher complexity. 

The main result of this note is the following theorem:
\begin{theorem}[ Refutation Complexity = Agnostic Learning Complexity, Informal] \label{thm:main-informal}
$\C$ has an efficient agnostic learning algorithm over a distribution $\cD$ with $m$ samples if and only if the refutation complexity of $\C$ at some polynomial running time is at most $O(m).$ 
\end{theorem}

\subsection{Comparison with \cite{vadhan2017learning}} 

In a recent, independent work, Vadhan \cite{vadhan2017learning} used similar arguments to establish a similar equivalence to Theorem \ref{thm:main-informal} between \emph{distribution independent PAC learning} in the realizable case (i.e. when the labels perfectly correlate with some concept in the target class) and a slightly different notion of refutation. In this notion, the refutation algorithm is required to distinguish the case that the sample that realizable (i.e., the labels agree with some concept from the class) from the case that the labels in the sample are i.i.d. Rademacher random variables.

Since agnostic learning is provably different from realizable PAC learning in general, the notions of refutation that characterize the  complexity of learning in the two models have to be necessarily different. Another interesting point of difference is that our equivalence is \emph{distribution-specific} and thus slightly more fine-grained in that it allows relating learnability on a given distribution to refutation on the same distribution. In contrast, Vadhan's characterization holds for distribution independent PAC learning. This difference arises entirely due to the the difference in the black-box boosting algorithms one can use in PAC vs agnostic settings\footnote{We thank Salil Vadhan for pointing this out to us.}: in the PAC learning case, the boosting algorithms modify the distribution of examples over the course of the execution and thus the characterization holds only in a distribution independent setting. In the agnostic setting, there are distribution specific boosting algorithms (such as that of \cite{DBLP:conf/nips/KalaiK09,feldman2010distribution}) that work by changing only the distributions of the labels while keeping the distribution of the example points unchanged. It is an interesting direction to investigate notions of refutation that allow \emph{distribution-specific} characterization of PAC learning in realizable case.

It's interesting to note how slight changes to in the formulation of the refutation problem changes the model of learning that it characterizes.


\subsection{Discussion}
\paragraph{Proper vs Improper Learning and the Framework of \cite{daniely2014average}} The agnostic learning algorithm we obtain using a refutation algorithm is \emph{improper} - that is, it doesn't necessarily produce a hypothesis from the class $\C$. This is not accidental - it's well known that the flexibility of \emph{improper} learning allows circumventing computational hardness results that afflict \emph{proper} learning. A simple example is the class of 3-term DNF formulas in $n$ variables: unless RP = NP, there's no polynomial time \emph{proper} learning algorithm for this class \cite{DBLP:journals/jacm/PittV88}, however, there's a simple $\poly(n,1/\epsilon)$-time \emph{improper} learning algorithm for it (for a discussion see, \cite{DBLP:journals/jmlr/Shalev-ShwartzST12}). On the flip side, the power of \emph{improper} learning makes the task of proving lower bound against such algorithms harder. The equivalence between refutation and agnostic learning holds for all (and thus, also improper) learning algorithms and thus can serve as a useful handle in understanding the complexity of improper learning.

Indeed this connection and in particular, the implication that learning implies refutation is implicit in the influential work of Daniely and co-authors \cite{DBLP:conf/stoc/Daniely16,daniely2016complexity,DBLP:conf/stoc/DanielyLS14} who showed (in the language of this paper) that a refutation algorithm for the concept classes of halfspaces and DNF formulas can be used to refute certain random constraint satisfaction problems \cite{MR3473335-Allen15,MR2121179-Feige02}. These works used such a reduction along with standard hardness assumptions for refuting random CSPs to obtain the first hardness results for improper PAC learning for the above classes. 

Our equivalence establishes the converse of the connection in these works and makes the connection between refutation and agnostic learning explicit. While a priori, it might appear that refutation (which asks for distinguishing between a pure noise in the labels from a correlated set of labels) is easier than agnostic learning, this work shows that any lower bound on (improper) learning has to necessarily be a lower bound for an associated refutation problem. Thus, to an extent, it shows that the above framework for improper agnostic learning lower bounds is essentially complete.




\paragraph{Connections to Boosting/Property Testing} It is also illuminating to view the equivalence we show as saying that an oracle for refutation is sufficient for agnostic learning. This naturally leads to the question of what kind of oracle access to $\C$ is sufficient for (agnostic) learning. We discuss two natural oracles here: a weak-learning oracle and a property-testing oracle.

Known boosting (see \cite{freund1995desicion, schapire90}, \cite{DBLP:conf/nips/KalaiK09, feldman2010distribution}) algorithms imply that a \emph{weak-learning oracle} is sufficient for agnostically learning of $\C$. A weak learning oracle takes random example-label pairs and returns a hypothesis whose correlation with the labels from the input distribution is at least an inverse polynomial fraction of the correlation of the best-fitting hypothesis from $\C$. In learning literature, this is sometimes referred to as a \emph{weak-optimization} oracle for $\C$ - in that, it gives a inverse polynomial (potentially improperly) approximation to the correlation of the best fitting hypothesis from $\C$. It is not hard to see that such an oracle is enough to solve the refutation problem and thus is a potentially stronger access to $\C$ than the refutation algorithm. 

Our result implies that an much weaker algorithm is enough to get an agnostic learning algorithm - the refutation oracle doesn't return any hypothesis, it ``merely'' distinguishes between the case that the labels are completely random and independent of the examples from the case that the labels come from some distribution that correlates with some concept in $\C$. 

It is also instructive to compare a refutation oracle (or a ``structure`` vs ''noise'' tester) for $\C$ with a ``property-tester'' for $\C$. An $\alpha$-approximate property-testing algorithm for $\C$ uses random example-label pairs \footnote{Property testers are usually defined with $\alpha = 1$ and are in general also allowed to use membership queries. We use a definition that is similar in spirit but is more relevant for the comparison here.} from some distribution and accepts if the labels achieve a correlation of at least $\alpha$ with $\C$ and rejects if tevery $c \in \C$ has a correlation of at most $\alpha-\epsilon$ with the labels. We can interpret a property tester, thus, as a variant of the refutation oracle that must treat a distribution on example-label pairs that has a correlation of at most $\alpha-\epsilon$ with every $c \in \C$ as ``noise.'' In particular, the notion of what is ``unstructured/noise'' for a property tester is more stringent compared to a refutation algorithm. Indeed, this is not surprising: while testing is known to be no harder than \emph{proper} learning, it can be harder than \emph{improper} learning for some concept classes, once again illustrating the difference between proper and improper learning \cite{DBLP:journals/jacm/GoldreichGR98}. 

\paragraph{Using Refutation to get Learning Algorithms} It will be extremely interesting to understand if the equivalence between refutation and learning allows an application in the direction opposite to the one employed in the work of Daniely and co-authors and get new algorithms for agnostic learning. This is perhaps not too optimistic. The works of Daniely and co-authors establish a natural connection between the refutation problem for a concept class and refuting random CSPs. There are known algorithms for refuting random CSPs (see for e.g. \cite{DBLP:journals/corr/RaghavendraRS16,DBLP:conf/focs/Feige07,MR3473335-Allen15}) that use techniques that appear different from the usual tool-kit in agnostic learning (for e.g. the use of semi-definite programming) that might prove useful in obtaining new agnostic learning algorithms by building the required refutation algorithms. 
\subsection{Proof Overview}
It is easy to see that efficient learning implies efficient refutation. For the other direction, we give an explicit, efficient algorithm that invokes the refutation algorithm a small number of times to get an agnostic learner for the class $\C$. This algorithm works in two steps - in the first step, it uses a refutation algorithm to come up with a \emph{weak-agnostic} learner: i.e. a hypothesis that achieves a correlation with the labels that is some tiny fraction of the correlation of the best hypothesis from $\C$. In the second step, it combined an off-the-shelf boosting algorithm with the weak learner above to get an agnostic learner with small error. 

The key idea in the transformation of a refutation algorithm into a weak-learner is to view the black-box refutation algorithm as a ``code'' for computing a function by manipulating the example-label pairs that it takes as input. A simple hybrid argument then shows that there's a small list of hypotheses generated by manipulating the inputs to the refutation algorithm that contains a good weak learner. We can find the best weak learner from the list by evaluating the error of each of the hypotheses in the list over a fresh batch of samples from the underlying distribution.
\subsection{Preliminaries}
We use $\U_m$ to denote the uniform distribution over $\on^m$ for any $m \in \N.$ 
We define agnostic learning here.
\begin{definition}[Agnostic Learning with respect to a distribution $\cD$] \label{def:agnostic-learning}
Let $\C$ be a class of Boolean concepts $\C \subseteq \{ f: \on^n \rightarrow \on \}$. $\C$ is said to be $\epsilon$-agnostically learnable in time $T(n,1/\epsilon)$ and samples $S(n,1/\epsilon)$ if there's an algorithm $\cA$ running in time $T(n,1/\epsilon)$ that takes $S(n,1/\epsilon)$ random labeled examples $\{ (x_i,y_i) \mid 1 \leq i \leq m\}$ where $(x_i,y_i)$s are i.i.d. from $\cD'$, such that the marginal on $x_i$ is $\cD$ and outputs with probability at least $3/4$, a hypothesis $h: \on^n \rightarrow \on$ such that $\E_{(x,y) \sim \cD'}\left[ \one{h(x)\ne y}\right] \leq \inf_{c \in \C} \E_{(x,y) \sim \cD'}\left[\one{c(x)\ne y}\right] + \epsilon.$ 
\end{definition}




\section{Refutation Complexity}
In this section, we define refutation complexity of a class of hypothesis with respect to a distribution $\cD$.

\begin{definition}[Refutation Algorithm for Distribution $\cD$]
Let $\C \subseteq \{f: \R^n \rightarrow \on\}$ be a class of Boolean concepts. Let $\cD$ be a distribution on $\R^n.$

A $\delta$-\emph{refutation} algorithm $\cA$ for $\C$ on $\cD$ with $m = m(n)$ samples is a (possibly randomized) algorithm that takes input an $m$-tuple of points $\{ x_1, x_2, \ldots, x_m \} \subseteq \on^n$ and an $m$-tuple of labels $(\sigma_1, \sigma_2, \ldots, \sigma_m) \in \on^m$ and outputs either $\noise$ or $\unclear$ with the following guarantees:
\begin{enumerate}
\item \textbf{Completeness: } If $\{(x_i, \sigma_i)\}_{i \leq m}$ are i.i.d. from a distribution $\cD'$ on $\R^n \otimes \{\pm 1\}$ such that the marginal on $\R^n$ equals $\cD$ and $ \sup_{c \in \C} \E_{(x,\sigma) \sim \cD'}[c(x) \sigma] \geq \delta$, then, 
\[
\Pr_{\begin{subarray}{c}
\{(x_i,y_i)\}_{i \leq m} \sim_{i.i.d.} \cD'\\
 \text{ internal randomness of } \cA \end{subarray}} [ \text{ output = } \unclear] \geq 2/3.
\]
\item \textbf{ Soundness: } 
\[
\Pr_{\begin{subarray}{c}
(\sigma_1, \sigma_2, \ldots, \sigma_m) \sim \U_m\\
x_1,x_2, \ldots, x_m \sim \cD\\
 \text{ internal randomness of } \cA \end{subarray}} \Pr_{ } [ \text{ output = } \noise] \geq 2/3.
 \] 
\end{enumerate}
\end{definition}

\begin{definition}[$\delta$-Refutation Complexity]
Let $\C \subseteq \{f: \on^n \rightarrow \on\}$ be a class of Boolean concepts. 
Let $\cD$ be a distribution on $\on^n.$

The $\delta$-\emph{refutation complexity} of $\C$ on a distribution $\cD$ with running time $T(n)$ denoted by $\cR_{T(n),\delta}(\C)$, is the smallest $m = m(n,\delta)$ such that there exists a $\delta$-refutation algorithm for $\C$ on $\cD$ running in time $T(n)$ and $m$-samples. When $T(n)$ is not stated explicitly, we assume $T(n) = \poly(n)$ for some fixed polynomial in $n$. 
\end{definition}
\begin{remark}
Observe that the refutation complexity, just as Rademacher complexity is distribution dependent. Further, for $T(n) = \infty$, $\delta$-refutation complexity degenerates into Rademacher complexity. At non-trivially bounded running times (of special interest, of course, is polynomial time algorithms), refutation complexity captures the sample complexity of \emph{efficient} agnostic, improper learning $\C$ over $\cD$ as we show next and thus can be much larger than the Rademacher complexity.
\end{remark}

\section{Learning vs Refutation Complexity}

In this section, we establish the equivalence between agnostic learning a class $\C$ over a given distribution $\cD$ and the refutation problem with  respect to the distribution $\cD$ for the concept class $\C$.

We begin by showing the Learning implies Refutation, which is the easy direction.

\begin{lemma}[Learning implies Refutation]
Suppose $\C$ is $\epsilon$-agnostically learnable in time $T(n,\epsilon)$ and samples $S(n,\epsilon)$ over the distribution $\cD$.
Then, the refutation complexity of $\C$ with respect to the distribution $\cD$ at the running time $T(n,\delta/4)$ is at most $2S(n,\delta/4) + 128/\delta^2.$
\end{lemma}
\begin{proof}
Let $m = S(n,\delta/4) + 64/\delta^2.$ 

The $\delta$-refutation algorithm gets input $x_1, x_2, \ldots, x_{2m}$ and $\sigma_1, \sigma_2, \ldots, \sigma_{2m}.$ It runs the $\epsilon$-agnostic learner on examples $\{(x_i, \sigma_i)\}_{i =1}^m$ for $\epsilon = \delta/4$ and obtains a hypothesis $h$. Let $\cor_h = \frac{1}{m}\sum_{i =m+1}^{2m} \sigma_i \cdot h(x_i).$ If $\cor_h \geq \delta/2$, output $\unclear$ otherwise output $\noise.$

We now analyze the completeness and the soundness properties of this algorithm. 

First, suppose $\{(x_i, \sigma_i)\}_{i \leq 2m}$ were i.i.d. according to some $\cD'$ such that the marginal on $\R^n$ equals $\cD$. Let $\cor_f(\cD') = \E_{(x,y) \sim \cD'}[ f(x) y].$ Then, with probability $2/3$ over the draw of the sample, the agnostic learner produces a hypothesis $h$ such that $\cor_h \geq \cor_h(\cD') - \epsilon \geq \cor_c(\cD')-2\epsilon$ for every $c \in \C$. Thus, if $\cor_c(\cD') \geq \delta$, then, $\cor_h \geq \delta - \epsilon/2 \geq \delta/2.$ Thus, in this case, the algorithm above outputs $\unclear$ as desired.

Now suppose $\sigma_i$s are i.i.d. Rademacher and independent of $x_i$s. Then, since $\sigma_{m+1}, \ldots, \sim_{2m}$ are independent of $\sigma_1, \ldots, \sigma_m$, $\cor_h \leq \frac{4}{\sqrt{m}} < \delta/2$ using that $m  > 64/\delta^2.$




\end{proof}

We now prove the other direction:


\begin{lemma}[Learning by Refutation] \label{lem:refutation-to-learning}
Suppose that the $\delta$-refutation complexity of a class of Boolean concepts $\C$ with respect to a distribution $\cD$ at a running time $T(n)$ is $m = \cR_{T(n),\delta}(\C).$ Then, there's an algorithm that runs in time $T(n) \frac{m^2}{\epsilon^2}$ and uses $ O(\frac{m^3}{\epsilon^2})$ samples to $(\delta+\epsilon)$-agnostically learn $\C$ on $\cD$.
\end{lemma}

The proof is in two steps. In the first step, we show that the refutation algorithm yields a weak agnostic learner for $\C$ with respect to the distribution $\cD$. In the second step, we use the distribution specific agnostic boosting algorithm (see \cite{DBLP:conf/nips/KalaiK09}) to boost the accuracy of the weak learner to obtain an agnostic learner. We start by defining a weak-agnostic learner :

\begin{definition}[Weak Agnostic Learner]
An $(\gamma,\alpha)$-weak agnostic learner for a Boolean concept class $\C$ over a distribution $\cD$ is an algorithm that takes input random examples from a distribution $\cD'$ on example-label pairs $(x,y)$ such that the marginal on $x$ is $\cD$ such that with probability at least $3/4$ over its random input outputs a (randomized) hypothesis $h:\on^n \rightarrow \on$ such that $\E_{(x,y) \sim \cD'}[ y \cdot h(x) ]\geq \gamma (\sup_{c \in \C} \E_{(x,y) \sim \cD'}[y \cdot c(x)]) - \alpha.$
\end{definition}

\begin{lemma}[Refutation to Weak Agnostic Learner]
\label{lem:reftolearn}
Suppose that the $\delta$-refutation complexity of a class of Boolean concepts $\C$ with respect to a distribution $\cD$ at a running time $T(n)$ is $m = \cR_{T(n),\delta}(\C).$ Then, there's an $(\gamma,\alpha)$-weak agnostic learner for $\C$ on distribution $\cD$ that runs in time $T(n)$ and samples $m(n)$ where $\alpha = \delta \cdot \gamma$, $\gamma = \frac{2}{3m}.$
\end{lemma}

We describe a natural class of candidates for a weak learner that come out of running the refutation algorithm on appropriately chosen hybrids of the distribution $\cD'$ and $\cD \times \U_1.$ We begin by defining a class of $2(m+2)$ different functions denoted by $W_{i,b}:\on^n \rightarrow \zo$ for $0 \leq i \leq m+1$ and $b \in \pm 1$ produced by taking these hybrids. Our weak learners will be a simple transformation of this class.

\newcommand{\algnametwo}{3.1 \xspace}
\begin{center}
\fbox{\begin{minipage}{6in} 
\begin{center}
\textbf{\algnametwo Hybrid Functions $W_{i,b}$ } 
\end{center}

\begin{description}
\item[Input:] $x \in \R^n$, $b \in \on$.
\item[Output:] $W_{i,b}(x) = z \in \on.$ 
\item[Operation:] ~
\begin{enumerate}
\item Draw $(x_{1}, \sigma_{1}), \ldots, (x_{i-1},\sigma_{i-1})$ i.i.d. from $\cD \times \U_1$.  Draw $(x_{i+1}, y_{i+1}), (x_{i+2}, y_{i+2}) \ldots, (x_{m}, y_{m})$ i.i.d. from $\cD'$. 
\item Run the $\delta$-refutation algorithm on input\\ $(x_1,\sigma_1), (x_2, \sigma_2), \ldots, (x_{i-1},\sigma_{i-1}), (x,b), (x_{i+1},y_{i+1}), \ldots, (x_{m}, y_m)$.
\item Let $W_{i,b} = 1$ if the refutation algorithm returns $\unclear$ and $0$ otherwise.
\end{enumerate}
\end{description}
\end{minipage}}
\end{center}

We make some simple observations about $W_{i,b}$ that will come handy in the argument below.

Observe that $W_{m+1,b}$ is the function that evaluates to 1 if the output of the refutation algorithm on examples drawn from $\cD$ and labels i.i.d Rademacher variables is $\unclear$. On the other hand, $W_{0,b}$ is the function obtained when the refutation algorithm is run on example-label pairs from $\cD$. Finally, observe that 
\begin{equation}
\label{eq:randomizing-planting} \E_{b \sim \U_1} \E[ W_{i,b}(x)] = \E_{ (x,y) \sim \cD} \E[W_{i+1,y}(x)]
\end{equation}
Here, the inside expectation is over all the random choices within the procedure for computing $W_{i,b}$s above. We can now present our candidate weak learners.

\paragraph{Candidate Weak Learners} For every $0 \leq i \leq m+1$, let $h_i(x) = W_{i,+1}(x) - W_{i,-1}(x).$

\begin{proof}[Proof of Lemma \ref{lem:reftolearn}]
Our weak learning algorithm is given access to random labeled examples from a distribution $\cD'$ on $\R^n \otimes \{ \pm 1\}.$ The weak learner will draw a sample from $\cD'$ of size $O(\log{m})$ from $\cD'$ and chooses the $h_i$ that has the maximum correlation with the labels. Observe that with $O(\log{(m)})$ samples, the correlations of $h_i$ on $\cD'$ will be faithfully preserved with $2/3$ probability. Thus, to complete the proof, we only need to argue that one of the $h_i$s is always an $(\alpha,\gamma)$-weak learner.

 To show this, we must argue that there exists an $0\leq i \leq m+1$ such that:

\[
\E_{(x,y) \sim \cD'}[ y \cdot h_i(x) ] \geq \frac{2}{3m} \sup_{c \in \C} \E_{(x,y) \sim \cD'}[ c(x) \cdot y] - \frac{2}{3m} \delta.
\]

Observe that the guarantees of the weak learner are trivial if $\sup_{c \in \C} \E_{(x,y) \sim \cD'}[ y \cdot c(x)] < \delta.$ Thus assume that $\sup_{c \in \C} \E_{(x,y) \sim \cD'}[ y \cdot c(x)] > \delta.$ In this case, we will show that $\E_{(x,y) \sim \cD'}[h_i(x) y] \geq \frac{2}{3m}\geq \frac{2}{3m} \sup_{c \in \C} \E_{(x,y) \sim \cD'}[ c(x) \cdot y] - \frac{2}{3m}\delta.$

Now, observe that over the randomness of both the refutation algorithm and over the draw of i.i.d. sample from $\cD'$ of size $m = S(n)$, $\E[ W_{0,b} (x)] \geq 2/3$ and $\E[W_{m+1,b}(x)] \leq 1/3$ for any $b$. Thus, $$\sum_{i = 0}^m \E[ W_{i,y}(x) - W_{i+1,y}(x)] \geq 1/3,$$ where the expectation is over the randomness in the draw $(x,y) \sim \cD'$ and over the randomness in $W_{i,y}$ for $0 \leq i \leq m+1.$

Thus, there must exist an $i$ such that $\E[ W_{i,y}(x) - W_{i+1,y}(x)] > 1/3m.$ Observe that by construction 
\begin{multline} W_{i,y}(x) = \frac{y+1}{2} \cdot W_{i,1}(x) - \frac{y-1}{2} W_{i,-1}(x) = y \cdot \frac{W_{i,1}(x) - W_{i,-1}(x)}{2} + \frac{1}{2} (W_{i,1}(x) + W_{i,-1}(x)) \\= \frac{1}{2} y \cdot h_i(x)  + \frac{1}{2} (W_{i,1}(x) + W_{i,-1}(x)). \label{eq:whatever}\end{multline}

Next, observe that by \eqref{eq:randomizing-planting}, $\E[\frac{1}{2} (W_{i,1}(x) + W_{i,-1}(x))] = \E[ W_{i+1,y}(x)].$ 
Taking expectations on both sides of \eqref{eq:whatever} and rearranging, we have: $\E[ y \cdot h_i(x)] \geq \frac{2}{3m}.$ 

This establishes that for $\gamma = \frac{2}{3m}$ and $\alpha = \delta \cdot \gamma$ our algorithm returns $(\alpha, \gamma)$-weak agnostic learner as desired.
\end{proof}

We can now use boosting to get a strong agnostic learner for $\C$ over $\cD$ by using the weak learning algorithm along with a boosting algorithm. Specifically, we will use the result of Kalai and Kanade \cite{DBLP:conf/nips/KalaiK09}  (see also \cite{feldman2010distribution}) who showed the following agnostic boosting algorithm that takes a $( \gamma, \alpha )$-weak learner and outputs a hypothesis whose error is competitive within $\alpha$ with respect to the best fitting hypothesis from the class $\C$. 

\begin{fact}[Agnostic Boosting \cite{DBLP:conf/nips/KalaiK09}] \label{fact:boosting}
Let $\C$ be a class of Boolean concepts.
Let $\cD$ be a distribution on $\on^n$ and $\epsilon > 0.$

There's an algorithm that takes random labeled examples from a distribution $\cD'$ on example-label pairs $(x,y)$ such that the marginal on $x$ is $\cD$, invokes a $(\gamma,\alpha)$-weak learner for $\C$ $O(\frac{1}{\gamma^2 \epsilon^2})$ times and outputs a hypothesis $h: \on^n \rightarrow \on$ such that $$\E_{(x,y) \sim \cD'}\left[ \one{h(x)\ne y} \right] \leq \inf_{c \in \C} \E_{(x,y) \sim \cD'}\left[\one{c(x)\ne y}\right] + \alpha/\gamma + \epsilon.$$
The algorithm needs $S(n) \cdot O(\frac{1}{\gamma^2 \epsilon^2})$ samples and runs in time $T(n) \cdot O(\frac{1}{\gamma^2 \epsilon^2})$ where $S(n)$ and $T(n)$ are the sample complexity and the running time respectively of the $(\gamma, \alpha)$-weak agnostic learner.
\end{fact}{}

We get Lemma \ref{lem:refutation-to-learning} as an immediate corollary of Fact \ref{fact:boosting} and Lemma \ref{lem:reftolearn}.

\subparagraph*{Acknowledgements.}
We thank Salil Vadhan for sharing an early version of \cite{vadhan2017learning} with us and illuminating follow-up discussions. P.K. thanks Avi Wigderson for many useful comments and suggestions about this work and David Steurer for helpful discussions on related problems. This research was supported by funding from Eric and Wendy Schmidt Fund for Strategic Innovation.


\bibliography{mathreview,dblp,scholar,custom,roi}

\end{document}